\documentclass[letterpaper, 10 pt, journal, twoside]{IEEEtran}
\usepackage[T1]{fontenc}
\usepackage[latin9]{inputenc}
\usepackage{xcolor}
\usepackage{verbatim}
\usepackage{float}
\usepackage{amsmath}

\usepackage[T1]{fontenc}
\usepackage[latin9]{inputenc}
\usepackage{color}
\usepackage{verbatim}
\usepackage{float}
\usepackage{amsmath}

\usepackage{amsthm}
\usepackage{amssymb}
\usepackage{stmaryrd}
\usepackage{stackrel}
\usepackage{graphicx}
\usepackage{wasysym}
\usepackage{mathtools}
\usepackage{hyperref}
\usepackage{url}   
\usepackage{booktabs} 
\usepackage{nicefrac} 
\usepackage{authblk}
\usepackage{multirow}
\usepackage{xcolor}
\usepackage{microtype}
\usepackage{paralist}

\makeatletter

\floatstyle{ruled}
\newfloat{algorithm}{tbp}{loa}
\providecommand{\algorithmname}{Algorithm}
\floatname{algorithm}{\protect\algorithmname}

\theoremstyle{plain}
\newtheorem{thm}{\protect\theoremname}
\theoremstyle{definition}
\newtheorem{defn}{\protect\definitionname}
\theoremstyle{definition}
\newtheorem{problem}{\protect\problemname}
\theoremstyle{plain}
\newtheorem{lem}{\protect\lemmaname}
\theoremstyle{definition}
\newtheorem{example}{\protect\examplename}
\theoremstyle{remark}
\newtheorem{rem}{\protect\remarkname}
\theoremstyle{plain}

\usepackage{cite}
\usepackage{algpseudocode}
\usepackage{setspace}
\IEEEoverridecommandlockouts

\usepackage{algpseudocode}
\usepackage{setspace}
\usepackage{color}
\usepackage{comment}
\usepackage{tikz}
\usepackage{xspace}
\tikzset{>=latex}
\usetikzlibrary{fit,automata,positioning,angles,quotes}

\makeatother

\usepackage{babel}
\providecommand{\corollaryname}{Corollary}
\providecommand{\definitionname}{Definition}
\providecommand{\examplename}{Example}
\providecommand{\lemmaname}{Lemma}
\providecommand{\problemname}{Problem}
\providecommand{\remarkname}{Remark}
\providecommand{\theoremname}{Theorem}

\newcommand{\UntilOp}{\mathcal{U}}
\newcommand{\Eventually}{\diamondsuit}
\newcommand{\Always}{\square}
\newcommand{\Next}{\Circle}
\newcommand{\AP}{{AP}}

\newcommand{\dynsys}{\ensuremath{\mathfrak{S}}\xspace}

\usepackage{amsmath}               
  {
      \theoremstyle{plain}
      
      \theoremstyle{plain}
      
  }
  
\SetSymbolFont{stmry}{bold}{U}{stmry}{m}{n}

\def\BibTeX{{\rm B\kern-.05em{\sc i\kern-.025em b}\kern-.08em
		T\kern-.1667em\lower.7ex\hbox{E}\kern-.125emX}}
		
\renewcommand{\baselinestretch}{0.92}

\begin{document}

\title{Learning Minimally-Violating Continuous Control for Infeasible Linear Temporal Logic Specifications

\thanks{DISTRIBUTION STATEMENT A. Approved for public release. Distribution is unlimited. This material is based upon work supported by the Under Secretary of Defense for Research and Engineering under Air Force Contract No. FA8702-15-D-0001. Any opinions, findings, conclusions or recommendations expressed in this material are those of the author(s) and do not necessarily reflect the views of the Under Secretary of Defense for Research and Engineering.}

\thanks{$^{1}$Mingyu Cai and Cristian-Ioan Vasile are with Mechanical Engineering, Lehigh University, Bethlehem, PA, 18015 USA.
        {\tt\footnotesize mic221@lehigh.edu,  
crv519@lehigh.edu}}
\thanks{$^{2}$Makai Mann, Zachary Serlin and Kevin Leahy are with MIT Lincoln Laboratory, Lexington, MA, 02421, USA.
        {\tt\footnotesize makai.mann@ll.mit.edu,  
Zachary.Serlin@ll.mit.edu, kevin.leahy@ll.mit.edu}}

}

\author{Mingyu Cai$^{1}$, Makai Mann$^{2}$, Zachary Serlin$^{2}$, Kevin Leahy$^{2}$, Cristian-Ioan Vasile$^{1}$
}

\markboth{2023 American Control Conference. Preprint Version. }
{Cai \MakeLowercase{\textit{et al.}}: Learning Minimally-Violating Continuous Control for Infeasible LTL Specifications} 

\maketitle

\begin{abstract}
This paper explores continuous-time control synthesis for target-driven navigation to satisfy complex high-level tasks expressed in linear temporal logic (LTL). We propose a model-free framework using deep reinforcement learning (DRL) where the underlying dynamical system is unknown (an opaque box). Unlike prior work, this paper considers scenarios where the given LTL specification might be infeasible and therefore cannot be accomplished globally. Instead of modifying the given LTL formula, we provide a general DRL-based approach to satisfy it with minimal violation. 
 To do this, we transform a previously multi-objective DRL problem, which requires simultaneous automata satisfaction and minimum violation cost, into a single objective. By guiding the DRL agent with a sampling-based path planning algorithm for the potentially infeasible LTL task, the proposed approach mitigates the myopic tendencies of DRL, which are often an issue when learning general LTL tasks that can have long or infinite horizons. This is achieved by decomposing an infeasible LTL formula into several reach-avoid sub-tasks with shorter horizons, which can be trained in a modular DRL architecture. Furthermore, we overcome the challenge of the
 exploration process for DRL in complex and cluttered environments by using path planners to design rewards that are dense in the configuration space. The benefits of the presented approach are demonstrated through testing on various complex nonlinear systems and compared with state-of-the-art baselines. The video demonstration can be found here:~\url{https://youtu.be/DqesqBsja9k}.
\end{abstract}

\begin{IEEEkeywords}
Formal Methods in Robotics and Automation, Deep Reinforcement Learning, Sampling-based Method
\end{IEEEkeywords}



\section{Introduction}

\global\long\def\Dist{\operatorname{Dist}}%
\global\long\def\Eval{\operatorname{Eval}}%

Autonomous agents operating in complex environments must often accomplish high-level tasks while subject to various uncertainties. Possible uncertainties include unknown obstacles or targets, and impassable or unexpected terrain. 
Additionally, a principled analytical model of the robot is typically (partially) unknown or hard to obtain. 
Operating in a complex, uncertain environment with unknown dynamics inevitably leads to deviations from the planned trajectory which can cause the original desired task to be infeasible. Ideally, a deployed system will fail gracefully when faced with an impossible task.

Motivated by these challenges, this work investigates model-free control of an autonomous system for satisfying a potentially infeasible high-level task with minimal violation while operating in a complex environment. 


In the field of model-free navigation control,
Reinforcement Learning (RL) is a popular unsupervised technique that optimizes long-term expected rewards to learn desired behaviors~\cite{sutton2018reinforcement}. 
Recently, Deep Reinforcement Learning (DRL) techniques, such as the actor-critic method~\cite{haarnoja2018soft}, have been shown to be capable of learning continuous controllers for unknown (opaque-box) dynamics. Existing works~\cite{long2018towards, chiang2019learning} apply DRL algorithms to solve standard goal-oriented problems, but does not explicitly investigate complex navigation tasks. 
Linear Temporal Logic (LTL) is a formal language that has gained attention in recent years for expressing high-level, complex robotics tasks. Specifications have been shown to be effective at directing RL agents to learn desired policies.
Typically within the literature, a discrete robot system is abstracted as a discrete Markov Decision Process (MDP) model and composed with an automaton representing the desired LTL formula to create a product automaton for learning or planning ~\cite{sadigh2014learning, kantaros2022accelerated, xu2020joint, cai2021reinforcement}.  This approach has been extended to continuous systems
\cite{li2019formal,icarte2022reward}, where the LTL formulas are only defined over finite horizons. For general continuous control subject to LTL satisfaction over infinite horizons, prior works~\cite{Cai2021modular, cai2021safe, cai2022overcoming} proposed a modular architecture by decomposing the task into sub-tasks using automata states. 
All this related work designs RL rewards according to the discrete switch of an automaton, which requires the RL agent to visit every region of interest to gain experience in each automaton state. Even relying on reward shaping to enrich the intermediate rewards, this paradigm can be extremely challenging for DRL in complex, cluttered environments. Primarily, these approaches are inefficient at exploring the environment, which is a critical component of performing 
RL~\cite{explorationref,hrl-survey}. Crucially, to our knowledge, no previous work has investigated infeasible LTL tasks for continuous control. In this work, we focus on the problem including two aspects:
infeasible LTL tasks and cluttered environments using DRL for model-free continuous control.

\textbf{Related works:\space}
In the control community, signal temporal logic (STL) and LTL are two popular formal languages. Maximum satisfaction of STL can be achieved via optimizing its robustness degrees~\cite{raman2014model,rodionova2022combined}.
Minimum violation  of LTL specifications over finite horizons has been
considered in~\cite{Tumova2013, Lahijanian2016, vasile2017minimum, niu2018minimum, Lacerda2019, niu2020optimal, rahmani2020you, wongpiromsarn2021minimum, KaKaVa-IROS-2021} using sampling-based methods, graph-based optimization algorithms, and dynamic programming to solve the corresponding motion planning problem.
For LTL specifications over infinite horizons, general violation measurements are proposed in \cite{kim2015minimal} by quantitatively revising LTL automata. The approaches are applied in~\cite{Guo2015,cai2020receding, cai2021optimal}, where they build product structures and encode potentially conflicting parts of a specification as a soft constraint. However, the prior literature generally assumes known dynamical systems with abstracted navigation controllers and only focuses on motion planning problems ~\cite{Tumova2013, Lahijanian2016, vasile2017minimum, Lacerda2019, niu2020optimal, wongpiromsarn2021minimum, Guo2015,cai2020receding}.
Our work improves on this by directly synthesizing a low-level controller for unknown dynamics in a model-free manner that minimally violate LTL satisfaction over both finite and infinite horizons.

Previous work~\cite{hasanbeig2019reinforcement, Cai2020} has tackled RL while considering infeasible tasks in discrete state-action spaces.
In particular, the authors in~\cite{hasanbeig2019reinforcement} use a limit-deterministic generalized B\"{u}chi automaton (LDGBA)~\cite{Sickert2016} whose acceptance condition requires visiting all accepting sets infinitely often.
This work accounts for the infeasible solution case by visiting the accepting sets as much as possible (if at all).
But many LTL formulas only have one accepting state in the LDGBA and this strategy often fails to generalize to those cases.
The learning objectives in~\cite{Cai2020} ensure satisfaction of the LDGBA acceptance condition and, thus, LTL formulae, and reduce the violation cost, resulting in a multi-objective RL (MORL) problem.
However, providing minimum-violation performance guarantees for MORL continues to be an open problem~\cite{van2014multi}.

Commonly in the literature for navigation in complex environments, motion planning algorithms are first used to find possible paths, and then path-tracking controllers are employed to follow that path ~\cite{qin2003survey, manchester2017control, chen2021fastrack}. This approach decomposes into a goal-reaching control problem and can be easily adapted for LTL satisfaction~\cite{luo2021abstraction,VaLiBe-IJRR-2020}. Classic approaches however require known dynamic models and may be computationally expensive for nonlinear and high-dimensional systems. This fundamentally motivates our focus on employing model-free control techniques.

\noindent \textbf{Contributions:} In this paper, we translate the MORL problem first analyzed in~\cite{Cai2020} into a standard DRL problem and apply model-free geometric motion planning to guide the learning process. The generated trajectories are shown to minimally violate a given infeasible LTL specification. By learning a controller to reach waypoints in a geometric path, our approach automatically obtains an optimal control policy for minimally-violating the infeasible LTL task specification. This paper's main contributions are: 
(i) To the best of our knowledge, this is the first model-free learning-based continuous control for potentially infeasible LTL satisfaction with minimum-violation objectives. Such objectives facilitate meaningful action for infeasible LTL tasks. Note that even though we focus on continuous control problems, the framework can be immediately extended to discrete systems by integrating with traditional tabular approaches;
(ii)  Our framework can learn continuous control in cluttered environments and mitigates the myopic tendencies of DRL by decomposing the global task into sub-tasks, which can be trained compositionally; 
(iii)  We demonstrate the benefits by comparing with several baselines in two nonlinear systems: a Dubins car and quadrotor.

\section{Preliminaries}
\label{sec:Preliminaries}

We use mathematical expressions for several notations, which are summarized in Table~\ref{tab:Notation}.

\begin{table}
	\caption{Summary of Notations.}
        \label{tab:Notation}
\centering{}\resizebox{0.48\textwidth}{!}{
	\begin{tabular}{|c|c|c|c|c|c|c}
	    \hline
		\textbf{Symbol} & \textbf{Notation}  \\ 
		\hline
		\text{$Proj:S\shortrightarrow X$} & \text{Projection between $S$ and $X$} \\
		\hline
		\text{$proj|_{X}:Q_{P}\shortrightarrow X$} & \text{Projection between $Q_P$ and $X$} \\
		\hline
		\text{$proj|_{Q}:Q_{P}\shortrightarrow Q$} & \text{Projection between $Q_P$ and$Q$} \\
		\hline
		\text{$L: S\shortrightarrow 2^{\AP}$} & \text{Label function of state space of dynamics } \\
		\hline
		\text{$L_X: X\shortrightarrow 2^{\AP}$} & \text{Label function of workspace} \\
		\hline
		\text{$L_P:Q_{P}\shortrightarrow 2^{\AP}$} & \text{Label function of relaxed PBA} \\
		\hline
  		\text{$S$} & \text{State space of an unknown dynamic system \dynsys} \\
            \hline
            \text{$X$} & \text{State space of a geometric workspace} \\
            \hline
            \text{$Q_P=X\times Q$} & \text{State space of a relaxed PBA} \\
            \hline
	\end{tabular}}
\end{table}

\textbf{Agent:\space}The general evolution of a continuous-time dynamical system \dynsys starting from an initial state $s_{0}\in S_{0}$ is given by

\begin{equation}
    \dot{s}=f\left(s, a\right),
    \label{eq:dynamics}
\end{equation}

where $s\in S\subseteq \mathbb{R}^{n}$ is the state vector in the compact set $S$, $a\in A\subseteq \mathbb{R}^{m}$ is the control input. 
We assume the flow field $f:\mathbb{R}^{n}\times\mathbb{R}^{m}\rightarrow\mathbb{R}^{n}$ is uniformly continuous in time and Lipschitz continuous in $s$ for fixed $a$. Under these assumptions there exists a unique solution of $\dynsys$ for a given $a$, providing trajectories of the system~\cite{Earl1955}. 
In this work, the function $f\left(s, a\right)$ is assumed to be unknown. 

\begin{defn}
\label{def:project}
For a robot operating in an environment $Env$, the geometric workspace can be represented by a compact subset $X\subset\mathbb{R}^{d}, d\in \left\{2,3\right\}$. 
The relation between dynamics $\dynsys$ and workspace $X$ is defined by the projection $Proj:S\shortrightarrow X$. 

The space $X$ contains regions of interest that are labeled by a set of atomic propositions $\AP$, with the labeling function $L_X:X\shortrightarrow2^{\AP}$. Let $L:S\shortrightarrow2^{\AP}$ be a labeling function over $\dynsys$ i.e., $L(s) = L_X(Proj(s))$.
\end{defn}

\noindent
\textbf{Reinforcement Learning:\ } The interactions between a robot with dynamics $\dynsys$ and an environment $Env$ can be captured as a continuous-labeled Markov Decision Process (cl-MDP). A cl-MDP is a tuple $\mathcal{M}=(S, S_{0}, A,p_{S}, R, \gamma, L)$,
where $S\subseteq\mathbb{R}^{n}$ is a continuous state space, $S_{0}$ is a set of initial states, $A\subseteq\mathbb{R}^{m}$
is a continuous action space, and $p_{S}$ captures the unknown system dynamics as a distribution.
The distribution $p_{S}:\mathfrak{B}\left(\mathbb{R}^{n}\right)\times A\times S\shortrightarrow\left[0,1\right]$
is a Borel-measurable conditional transition kernel, s.t. $p_{S}\left(\left.\cdot\right|s,a\right)$
is a probability measure of the next state given current $s\in S$ and $a\in A$ over the Borel
space $\left(\mathbb{R}^{n},\mathfrak{B}\left(\mathbb{R}^{n}\right)\right)$,
where $\mathfrak{B}\left(\mathbb{R}^{n}\right)$ is the set of all
Borel sets on $\mathbb{R}^{n}$. $R:S\times A\times S\shortrightarrow\mathbb{R}$ is the reward function, and $\gamma\in(0,1)$ is the discount factor. L is the labeling function in Def.~\ref{def:project}.

Since the dynamics $\dynsys$ are an opaque-box, the transition relation $p_{S}\left(\left.\cdot\right|s,a\right)$ for any state-action pair is unknown. Actor-critic RL
algorithms~\cite{haarnoja2018soft} have been demonstrated as promising tools to solve continuous-control problems for this cl-MDP model, where each valid transition of cl-MDP follows $f$ defined in~\eqref{eq:dynamics} that is zero-order hold model for a continuous-time action.

\begin{rem}
Since the cl-MDP model has continuous state and action space, it's intractable to construct it explicitly and synthesize the standard model-checking algorithm~\cite{baier2008}. Inspired by prior work~\cite{Cai2021modular}, we track it on-the-fly vie a deep neural network.
\end{rem}

Let $\Pi$ denote the set of all policies for a cl-MDP, and $\pi\in\Pi$ denote a stochastic policy $\pi: S\times A\shortrightarrow [0,1]$, that maps states to distributions over actions.
During each episode and starting from the initial state $s_{0}\in S_{0}$, the agent observes the state $s_{t}\in S$ and executes an action $a_{t}\in A$ at each time step $t$, according to the policy $\pi(a_{t}| s_{t})$, and $Env$ returns the next state $s_{t+1}$ sampled from $p_{S}(s_{t+1}|s_{t}, a_{t})$. The process is repeated until the episode is terminated.
The objective is to learn an optimal policy $\pi^{*}(a| s)$ that maximizes the expected discounted return $J(\pi)=\mathbb{E}^{\pi}\left[\stackrel[k=0]{\infty}{\sum}\gamma^{k}\cdot R(s_{k},a_{k},s_{k+1})\right]$.

\textbf{Linear Temporal Logic (LTL):\space}LTL is a formal language to describe complex properties and high-level specifications of a system. LTL formulae are built inductively from atomic propositions by applying Boolean and temporal operators to subformulae. 
The syntax in Backus-Naur form is given by:

\begin{equation*}
        \phi ::=  \top \mid \mu \mid \phi_1 \land \phi_2 \mid \lnot \phi_1 \mid \Next\phi \mid \phi_1 \UntilOp \phi_2 \mid  \Eventually\phi \mid \Always\phi,
\end{equation*}

where $\mu\in\AP$ is an atomic proposition, \emph{true} $\top$, \emph{negation} $\lnot$, and \emph{conjunction} $\land$ are propositional logic operators, and \emph{next} $\Next$ and \emph{until} $\UntilOp$ are temporal operators~\cite{baier2008}.
Alongside the standard operators introduced above, other propositional logic operators, such as \emph{false}, \emph{disjunction} $\lor$, and \emph{implication} $\rightarrow$, and temporal operators, such as \emph{always} $\Always$ and \emph{eventually} $\Eventually$, are derived from the standard operators.

The semantics of an LTL formula are interpreted over words, where a word is an
infinite (continuous) sequence  $o=o_{0}o_{1}\ldots$, with $o_{i}\in2^{\AP}$ for
all $i\geq0$, where $2^{\AP}$ represents the power set of $\AP$.
The satisfaction of an LTL formula $\phi$ by the word $o$ is denoted $o\models\phi$.
The detailed semantics of LTL can be found in~\cite{baier2008}.
In this paper, we are interested in continuous-time control policies. The \emph{next} operator is not always meaningful since it may require an immediate execution switch in the synthesized plans~space~\cite{kloetzer2008fully}. Based on the application of continuous space, we can either exclude the \emph{next} operator as is common in related work~\cite{kloetzer2008fully, luo2021abstraction} or properly design practical LTL tasks~\cite{Cai2021modular} in continuous scenarios.

\section{PROBLEM FORMULATION}

Consider a cl-MDP $\mathcal{M}=(S, S_{0}, A,p_{S}, R, \gamma, L)$.
The induced path under a policy $\pi$ over $\mathcal{M}$ is $\boldsymbol{s}_{\infty}^{\pi}=s_{0}\ldots s_{i}s_{i+1}\ldots$, where $p_{S}(s_{i+1}|s_{i}, a_{i})>0$ if $\pi(a_{i}| s_{i})>0$.

We extend the labeling function to traces such that $L\left(\boldsymbol{s}_{\infty}^{\pi}\right)=o_{0}o_{1}\ldots$
is the sequence of labels associated with $\boldsymbol{s}_{\infty}^{\pi}$.
We denote the satisfaction relation of the induced 
trace with respect to $\phi$ by $L(\boldsymbol{s}_{\infty}^{\pi})\models\phi$. The probability of satisfying $\phi$ under the policy $\pi$, starting from
an initial state $s_{0}\in S_{0}$, is defined as

\[
{\Pr{}_{M}^{\pi}(\phi)=\Pr{}_{M}^{\pi}(L(\boldsymbol{s}_{\infty}^{\pi})\models\phi\,\big|\,\boldsymbol{s}_{\infty}^{\pi}\in\boldsymbol{S}_{\infty}^{\pi}),}
\]
where $\boldsymbol{S}_{\infty}^{\pi}$ is the set of admissible
paths from the initial state $s_{0}$, under the policy ${\pi}$~\cite{baier2008}.

\begin{defn}
\label{def:feasibility}	
Given a cl-MDP $\mathcal{M}$, an LTL task
$\phi$ is fully feasible if and only if there exists a policy $\pi$ s.t. $\Pr{}_{M}^{\pi}(\phi)>0$. 
\end{defn}

Note that according to Def.~\ref{def:feasibility}, an infeasible case means no policy exists that satisfies the task, which can be interpreted as $\Pr{}_{M}^{\pi}(\phi)=0, \forall\pi\in\Pi$.

\begin{defn}
\label{def:infeasibility}	
Given a cl-MDP $\mathcal{M}$, the expected discounted violation cost with respect to a given LTL task $\phi$ under the policy $\pi$ is defined as
{
\begin{equation}
{J_{V}}(\mathcal{M}^{\pi},\phi)=\mathbb{E}^{\pi}\left[\stackrel[k=0]{\infty}{\sum} c_{V}(s_{i}, a_{i},s_{i+1},\phi)\right],\label{eq:violation}
\end{equation}} %
where $c_{V}(s,a,s',\phi)$ is formally defined in Def.~\ref{def:violation} as the violation cost of a transition $(s,a,s')$ with respect to $\phi$, and $a_{i}$ is the action generated based on the policy $\pi(s_{i})$.
\end{defn}

\begin{figure}[!t]
\centering
\centerline{\includegraphics[width=\columnwidth]{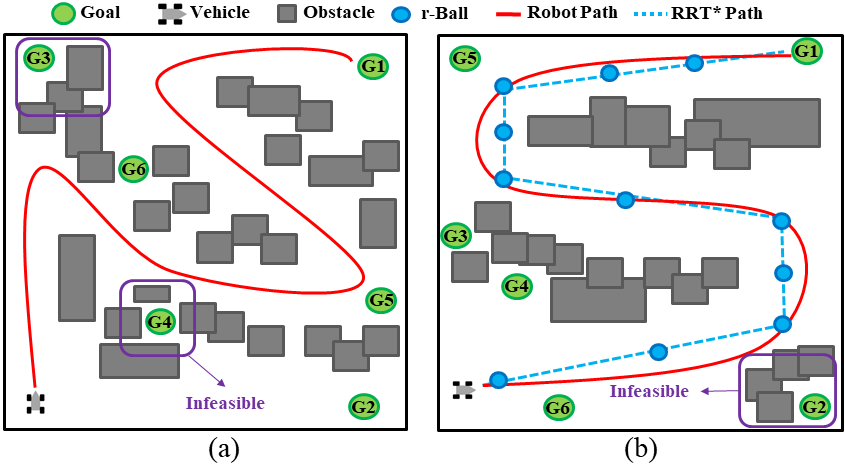}}
\caption{The figure shows the example described in Example~\ref{example1}, where the given LTL task is infeasible since some regions of interest, e.g., $\mathcal{G}_{3}, \mathcal{G}_{4}$ of (a). (a) The illustration that learning navigation using DRL is challenging in complex cluttered environments. (b) Demonstration of the effective design reward in Section~\ref{subsec:RL}.
}
\label{fig:running_example}

\end{figure}

\begin{problem}
\label{Prob1}Given a robot with unknown dynamics $\dynsys$ in an $Env$ with workspace $X$ containing  regions of interest, a navigation task in the form of an LTL task $\phi$ over $\AP$, their interactions can be captured as a cl-MDP.
The DRL objective is to find an optimal policy $\pi^{*}$ with the following capabilities: (i) if $\phi$ is feasible, $\Pr{}_{M}^{\pi^{*}}(\phi)>0$; (ii) if $\phi$ is infeasible s.t. $\Pr{}_{M}^{\pi}(\phi)=0, \forall\pi\in\Pi$, satisfy $\phi$ with minimum violation cost via minimizing $J_{V}(\mathcal{M}^{\pi},\phi)$.
\end{problem}



\begin{example}
\label{example1}
In Fig.~\ref{fig:running_example}, consider an autonomous vehicle with unknown dynamics deployed in a complex, cluttered environment containing a set of labeled goal regions, $\AP_{\mathcal{G}}=\left\{\mathcal{G}_1, \mathcal{G}_2,\ldots, \mathcal{G}_5\right\}$, and labeled obstacles $\mathcal{O}$. The LTL specification is given by $\phi=\Always\lnot\mathcal{O}\land\Always(\left(\Eventually\mathtt{\mathcal{G}_{1}}\land\Eventually\mathtt{\left(\mathcal{G}_{2}\land\Eventually\mathtt{\ldots\land \Eventually\mathcal{G}_{5}}\right)}\right)$. We observe that some goal regions are surrounded by obstacles such that $\phi$ is infeasible.

Furthermore, DRL-based navigation control typically uses either a sparse reward function that assigns positive rewards only when the robot reaches a destination, or leverages distance-to-goal measurements for a smoother reward. For LTL specifications, discrete goal-reaching may enable the automaton transitions, which is one requirement to receive automata-based rewards~\cite{li2019formal, Cai2021modular}. 
However, in cluttered environments such as those depicted in Fig.~\ref{fig:running_example}, it can be difficult for noisy policies to reach regions of interest at all during learning.
\end{example}

\section{SOLUTION}

In section~\ref{subsec:relaxed_PBA}, we develop an automaton model with relaxed constraints to address infeasible tasks.
In section~\ref{subsec:minimum}, we show how to synthesize a model-free plan with minimum violation cost in continuous space.
Finally, in section~\ref{subsec:RL}, we propose a novel DRL design to learn optimal policies that solve Problem~\ref{Prob1}.

\subsection{Relaxed Product Automaton for Infeasible LTL\label{subsec:relaxed_PBA}}

Let $dist: X\times X\rightarrow [0,\infty)$ define a  metric function that computes the geometric Euclidean distance between two states.
Prior work~\cite{cai2022overcoming} introduced a transition system for an unknown dynamical system to capture the interactions between geometric space  $X$ and $Env$.

\begin{defn}
\label{def:WTS}
A generalized weighted transition system (G-WTS) of $Env$ is a tuple $\mathcal{T}=(X, x_{0}, \rightarrow_{\mathcal{T}}, \AP, L_{X}, C_{\mathcal{T}})$, where $X$ is the configuration space of $Env$, $x_{0}$ is the initial state of the robot; $\rightarrow_{\mathcal{T}}\subseteq X\times X$ is the geometric transition relation s.t. $x\rightarrow_{\mathcal{T}}x'$ if $dist(x, x')\leq\eta$ and the straight line $\sigma$ connecting $x$ to $x_{new}$ is collision-free; $\AP$ is the set of atomic propositions that label regions; $L_X: X\rightarrow 2^\AP$ is the labeling function that returns an atomic proposition satisfied at location $x$; and $C_{\mathcal{T}}: (\rightarrow_{\mathcal{T}})\rightarrow\mathbb{R}^{+}$ is the geometric Euclidean distance, i.e.; $C_{\mathcal{T}}(x,x')=dist(x,x'),\forall (x,x')\in\rightarrow_{\mathcal{T}}$.
\end{defn}

Let $\tau_{\mathcal{T}}=x_{0}x_{1}x_{2}\dots$ denote a valid run of $\mathcal{T}$. As opposed to the standard WTS~\cite{kloetzer2008fully,luo2021abstraction} which includes dynamic state space $S$ and explicit dynamics $f(s,a)$ in transitions, only the geometric relations are available in our framework.

An LTL formula $\phi$ can be converted to a Non-deterministic B\"uchi Automaton (NBA) to evaluate satisfaction.

\begin{defn}
\label{def:NBA}
\cite{vardi1986automata} An NBA over $2^{\AP}$ is a tuple $\mathcal{B}=(Q, Q_{0}, \Sigma, \rightarrow_{\mathcal{B}}, Q_{F})$, where $Q$ is the set of states, $Q_{0}\subseteq Q$ is the set of initial states, $\Sigma=2^{\AP}$ is the finite alphabet, $\rightarrow_{\mathcal{B}}\subseteq Q\times\Sigma\times Q$ is the transition relation, and $Q_{F}\subseteq Q$ is the set of accepting states.
\end{defn}

A valid infinite run $\tau_{\mathcal{B}}=q_{0}q_{1}q_{2}\ldots$ of $\mathcal{B}$ is called accepting, if it intersects with $Q_{F}$ infinitely often. Infinite words $\tau_{o}=o_{0}o_{1}o_{2}\ldots, \forall o\in 2^{\AP}$ generated from an accepting run satisfy the corresponding LTL formula $\phi$. An LTL formula is converted into NBA using the \texttt{LTL2BA}\footnote{\url{http://www.lsv.fr/~gastin/ltl2ba/}}. We use $\mathcal{B}_{\phi}$ to denote the NBA of LTL formula $\phi$.

A common approach for synthesizing a geometric plan satisfying LTL formula $\phi$ is to construct the product B\"uchi automaton (PBA) between the G-WTS $\mathcal{T}$ and the NBA $\mathcal{B}$~\cite{kloetzer2008fully,luo2021abstraction,cai2022overcoming}. This approach assumes the given $\phi$ is feasible. Inspired by~\cite{kim2015minimal, Guo2015} to handle infeasible cases, we introduce a relaxed PBA and define corresponding violation costs.

\begin{defn}
\label{def:PBA}
Given the G-WTS $\mathcal{T}$ and the NBA $\mathcal{B}_{\phi}$, the relaxed PBA is a tuple $P=\mathcal{T}\times\mathcal{B}=(Q_P, Q^{0}_{P}, \rightarrow_{P}, Q^{F}_{P}, c_{P}, L_{P}, \Sigma)$,
where 
\begin{compactitem}[$\bullet$]
    \item $Q_{P}=X\times Q$ is the set of infinite product states, $Q^{0}_{P}=x_{0}\times Q$ is the set of initial states, $Q^{F}_{P}=X\times Q_{F}$ is the set of accepting states, and $\Sigma = 2^\AP$.
    \item $\rightarrow_{P}\subseteq Q_{P} \times Q_{P}$ is the transition function such that ${q_{P}=(x,q)\rightarrow_{P} q_{P}'=(x',q')}$ is a valid transition if and only if the following two conditions hold: $x\rightarrow_{\mathcal{T}}x'$ and $\exists \sigma_{\AP}\in\Sigma \text{ s.t. } q\overset{\sigma_{\AP}}{\rightarrow_{\mathcal{B}}}_{\phi}q'$,
    \item $c_{P}\colon (\rightarrow_{P})\rightarrow\mathbb{R}^{+}$ is the geometric cost function defined as the cost in the configuration space, e.g., $c_{P}(q_{p}=(x,q),q_{p}'=(x',q'))=C_{\mathcal{T}}(x,x'), \forall (q_{P},q_{P}')\in\rightarrow_{P}$, 
    \item $L_{P}\colon Q_{P}\to \Sigma$ is the labelling function s.t. $L_P(q_{P})=L_X(x), \forall q_{P}=(x,q)$,
    \item $c_{V}\colon (\rightarrow_{P})\rightarrow\mathbb{R}^{+}$ is the violation cost function for each transition with respect to the associated LTL formula $\phi$.
\end{compactitem}
\end{defn}

\noindent
The major novelty of the relaxed PBA is replacing the transition constraint $q\overset{L_X(x)}{\rightarrow_{\mathcal{B}_{\phi}}}q'$ of a standard PBA~\cite{baier2008}
with $\exists \sigma_{\AP}\in\Sigma \text{ s.t. } q\overset{\sigma_{\AP}}{\rightarrow_{\mathcal{B}}}_{\phi}q'$, to construct a more connected product graph. Next, we define the computation of the violation function $c_{V}$.


Suppose $\AP=\left\{\mu_{1}, \mu_{2},\ldots, \mu_{M}\right\}$
and consider an evaluation function $\Eval\colon\Sigma\shortrightarrow\left\{ 0,1\right\} ^{M}$, where $\Sigma = 2^\AP$ and $M=|\AP|$.
$\Eval(\sigma)=\left[v_{i}\right]^{M}_{i=1}$ and $v_{i}$ is a binary variable s.t. $v_{i}=1$
if $\mu_{i}\in \sigma$ and $v_{i}=0$ if $\mu_{i}\notin \sigma$, where
$i=1,2,\ldots,M$ and $\sigma\in\Sigma$. To quantify the difference between
two elements in $\Sigma$, consider $\rho\left(\sigma, \sigma'\right)=\left\Vert v-v^{\prime}\right\Vert _{1}=\sum_{i=1}^{M}\left|v_{i}-v_{i}^{\prime}\right|,$
where $v=\Eval\left(\sigma\right)$, $v^{\prime}=\Eval\left(\sigma^{\prime}\right)$,
$\sigma,\sigma^{\prime}\in\Sigma$, and $\left\Vert \cdot\right\Vert _{1}$
is the $L_{1}$ norm. The distance from $\sigma\in\Sigma$ to a set $\mathcal{X}\subseteq2^{\Sigma}$
is then defined as $D_{V}\left(\sigma,\ensuremath{\mathcal{X}}\right)=\underset{\sigma'\in\ensuremath{\mathcal{X}}}{\min}\rho\left(\sigma,\sigma'\right)$
if $\text{ }\sigma\notin\ensuremath{\mathcal{X}}$, and $D_{V}\left(\sigma,\ensuremath{\mathcal{X}}\right)=0$
otherwise.

\begin{defn}
\label{def:violation_cost}
In a relaxed PBA,
the violation cost of the transition ${q_{P}=(x,q)\rightarrow_{P} q_{P}'=(x',q')}$  imposed by $\phi$ can be
computed with $c_{V}(q_{P}, q_{P}')=D_{V}(L_P(q_{P}), \ensuremath{\mathcal{X}}(q, q'))$, 
where $\ensuremath{\mathcal{X}}(q, q')=\left\{ \sigma_{AP}\in\Sigma \mid q\overset{\sigma_{\AP}}{\rightarrow_{\mathcal{B}_{\phi}}}q'\right\}$ is the set of input symbols enabling the transition from $q$ to $q'$ in $\mathcal{B}_{\phi}$.
\end{defn}

Based on that, we can also measure the violation cost imposed on the dynamic system \dynsys as follows.

\begin{defn}
\label{def:violation}
Given a transition $(s,a,s')$ of dynamics and current automata state $q$, the next automata state $q'$ can be generated on-the-fly as  $q\overset{L(s)}{\rightarrow_{\mathcal{B}}}_{\phi}q'$. The violation cost of the transition with respect to $\phi$ can be obtained as $c_{V}(s,a,s',\phi) = D_{V}(L(s), \ensuremath{\mathcal{X}}(q, q'))$.
\end{defn}

\begin{rem}
Note that directly applying $c_{V}(s, a,s',\phi)$ into reward design for LTL satisfaction formulates a multi-objective DRL problem. Since DRL is a nonlinear regression process, the optimality performance of multi-objective DRL is hard to control. Instead, we use path planning over workspace $X$ for DRL guidance.    
\end{rem}

Given a valid trace $\tau_{P}=q^{0}_{P}q^{1}_{P}q^{2}_{P}\ldots$ of a PBA, the total violation cost can be computed with ${W_{V}}(\tau_{P})=\stackrel[i=0]{\infty}{\sum} c_{V}(q^{i}_{P},q^{i+1}_{P})$. The trace $\tau_{P}=q^{0}_{P}q^{1}_{P}q^{2}_{P}\ldots$ satisfies the acceptance conditions if it intersects with $Q^{F}_{P}$ infinitely often. Its corresponding words are obtained as $\tau_{o}=o_{0}o_{1}o_{2}\ldots$, $\forall o_{i}=L_{P}(q^{i}_{P})$ 
Now, the LTL satisfaction according to the accumulated violation cost can be measured as follows:

\begin{lem}
\label{lem:least-violation}
\cite{kim2015minimal}
    The words $\tau_{o}^{*}$ of an optimal trace $\tau_{P}^{*}$ satisfy the corresponding potentially infeasible LTL specification $\phi$ with minimum-violation guarantees if and only the following two conditions hold: (i) $\tau_{P}^{*}$ satisfies the acceptance condition and (ii) $W_{V}(\tau_{P}^{*})$ has the minimum total violation cost, i.e., $\tau_{P}^{*}=\underset{\tau_{P}\in\boldsymbol{\tau}_{\infty}}{\arg\min} \,W_{V}(\tau_{P})$, where $\boldsymbol{\tau}_{\infty}$ denotes the set of all valid traces in the relaxed PBA $P$.
\end{lem}

The word $\tau_{o}^{*}$ of $\tau_{P}^{*}$ satisfies $\phi$ exactly if $W_{V}(\tau_{P}^{*})=0$. Let $proj|_{X}: Q_{P}\rightarrow X$ denote a projection s.t. $proj|_{X}(q_{p})=x, \forall q_{P}=(x,q)$. We use this operator to extract the geometric trajectory $\tau_{\mathcal{T}}^{*}=proj|_{X}(\tau_{P}^{*})$ for minimally-violating LTL satisfaction from the optimal trace $\tau_{P}^{*}$.

\begin{rem}
    This work considers unknown dynamic systems so that the generated path $\tau_{\mathcal{T}}^{*}$ is only optimal in the sense of geometric relations e.g., shortest euclidean distance.
\end{rem}

\subsection{Minimum-violation Synthesis and Decomposition\label{subsec:minimum}}

\textbf{Minimum-violation Synthesis:\space} Since this work considers continuous control and the state space $X$ of $\mathcal{T}$ is also continuous, it's impossible to explicitly construct the relaxed PBA, $P$, and find the optimal trace $\tau_{P}^{*}$. Instead of discretizing space, we apply Temporal Logic Rapidly-Exploring Random Trees (TL-RRT*)~\cite{luo2021abstraction}, which is abstraction-free and builds a tree in the configuration space incrementally. This technique has the same properties as RRT*~\cite{karaman2011sampling,VaLiBe-IJRR-2020}: it can generate probabilistically complete and asymptotically optimal solutions.

Formally, let $\tau_{F}$ denote any accepting run of PBA $P$. TL-RRT* leverages the fact that an accepting run $\tau_{F}$ is a lasso-type sequence with a prefix-suffix structure, i.e., $\tau_{F}=\tau^{pre}_{P}[\tau^{suf}_{P}]^{\omega}$, where the prefix part $\tau^{pre}_{P}=q^{0}_{P}q^{1}_{P}\ldots q^{K}_{P}$ is only executed once, and the suffix part $\tau^{suf}_{P}=q^{K+1}_{P}\ldots q^{K+l}_{P}$ with $q^{K+1}_{P}=q^{K+l}_{P}$ loops infinitely.
Following this idea, we build the prefix and suffix optimal trees, respectively. To ensure the acceptance condition, the set of goal states of the prefix tree $G^{pre}_{P}=(V^{pre}_{P},E^{pre}_{P})$ is defined as $Q^{pre}_{goal} = \left\{q_{P}=(x,q)\in X_{free}\times Q\subseteq Q_{P}\mid q\in Q_{F}\right\}$, where $X_{free}$ is the collision-free configuration space. The optimal goal states of the prefix tree are $Q^{*}_{goal}=V^{pre}_{P}\cap  Q^{pre}_{goal}$ that is regarded as the roots of suffix tree $G^{suf}_{P}=(V^{suf}_{P},E^{suf}_{P})$. The destination states of the suffix tree are 
\begin{equation*}
\begin{array}{c}
    Q^{suf}_{goal}(q^{*}_{P})=\left\{\right.q_{P}=(x,q)\in X_{free}\times Q\subseteq Q_{P}\mid \\
    q_{P}\rightarrow_{P}  q_{p}^{*}=(x^{*}, q^{*}), \forall q_{p}^{*}\in Q^{*}_{goal}\left\}\right..
\end{array}
\end{equation*}
We refer readers to~\cite{luo2021abstraction} for more details.

Unlike \cite{luo2021abstraction}, this work considers infeasible cases where rapidly-exploring random trees should be built on-the-fly based on the relaxed PBA and
our optimization objective includes both violation and geometric cost.
To  guarantee total minimum-violation, let the weight between any two product states in the tree be $w_{P}(q_{P}, q_{P}')=c_{P}(q_{P}, q_{P}')+\beta\cdot c_{V}(q_{P}, q_{P}')$, where  we select $\beta \gg 1$ to ensure the violation cost $c_{V}$ has higher priority than geometric distance $c_{P}$. Finally, we can find the minimum-violation path in the form of $\tau^{*}_{P}=\tau^{*}_{pre}[\tau^{*}_{suf}]^{\omega}$, which can be projected into the configuration space $\tau^{*}_{\mathcal{T}}=proj|_{X}(\tau^{*}_{P})= \tau^{*}_{\mathcal{T},pre}[\tau^{*}_{\mathcal{T},suf}]^{\omega}$.

\textbf{Optimal decomposition:\space} Since directly applying the whole geometric path $\tau^{*}_{\mathcal{T}}=proj|_{X}(\tau^{*}_{P})$ with long or infinite horizons as guidance for DRL suffers from myopic tendencies (gathering intermediate rewards rather than reaching the desired region), we can decompose $\tau^{*}$ into several sub-paths. To do so, 
we rewrite $\tau^{*}_{P}=\tau^{*}_{0}\tau^{*}_{1}\ldots\tau^{*}_{K}[\tau^{*}_{K+1}\ldots \tau^{*}_{K+l}]^{\omega}$ such that $\forall i\in\left\{0,1,\ldots K+l\right\}$, $\tau^{*}_{i}$ is a sub-trajectory and the product states of each $\tau^{*}_{i}$ have the same automaton components i.e., $q=q_{i}, \forall q_{P}=(x, q)\in\tau^{*}_{i}$. Each segment $\tau^{*}_{i}$ can be projected into workspace as $\tau^{*}_{\mathcal{T},i}$, and we have 
   $$ \tau^{*}_{\mathcal{T}}=proj|_{X}(\tau^{*}_{P})=\tau^{*}_{\mathcal{T}, 0}\tau^{*}_{\mathcal{T}, 1}\ldots\tau^{*}_{\mathcal{T}, K}[\tau^{*}_{\mathcal{T}, K+1}\ldots \tau^{*}_{\mathcal{T}, K+l}]^{\omega}.$$

Let $proj|_{Q}: Q_{P}\rightarrow Q$ denote a projection s.t. $proj|_{Q}(q_{p})=q, \forall q_{P}=(x,q)$ to extract the automaton components, and we have
\begin{equation}
    \tau^{*}_{Q}=proj|_{Q}(\tau^{*}_{P})=q^{*}_{0}q^{*}_{1}\ldots q^{*}_{K}[q^{*}_{K+1}\ldots q^{*}_{K+l}]^{\omega}.\label{eq:automaton}
\end{equation}

As a result, we can define each segment $\tau^{*}_{\mathcal{T}, i}$ as a reach-avoid path $\mathcal{R}_{i}(\mathcal{G}_i,\mathcal{O})$, where $\mathcal{G}_i$ is its goal region and $\mathcal{O}$ is a set of all obstacles. The lasso-type geometric reach-avoid path can be written 
\begin{equation}
    \mathcal{R}_{\mathcal{T}}=(\mathcal{R}_{0}\mathcal{R}_{1}\ldots \mathcal{R}_{K})(\mathcal{R}_{K+1}\ldots \mathcal{R}_{K+l})^{\omega}.\label{eq:decomposition}
\end{equation}

For each reach-avoid task $\mathcal{R}_{i}$, let $x_{\mathcal{R},i}$ denote the destination state, which is also the last state of the segment $\tau^{*}_{\mathcal{T}, i}$. 
From definition~\ref{def:violation_cost}, we 
can obtain:

\begin{lem}
\label{lem:discrete_violation}
For a relaxed PBA, for a continuous trace $\tau_{P}=q^{0}_{P}q^{1}_{P}q^{2}_{P}\ldots$ satisfying the acceptance condition, we can extract an automaton trace  $\tau_{Q}=proj|_{Q}(\tau_{P})=$ $q_{0}q_{1}\ldots q_{K}[q_{K+1}\ldots q^{*}_{K+l}]^{\omega}.$ that satisfies the acceptance condition of the NBA.
The total violation cost {\small ${W_{V}}(\tau_{P})=\stackrel[k=0]{\infty}{\sum} c_{V}(q^{k}_{P},q^{k+1}_{P})$} can be transformed into the finite discrete form using its automaton trace as {\small $\stackrel[k=0]{K-1}{\sum}D_{V}(L_{X}(x_{\mathcal{R},k+1}), \ensuremath{\mathcal{X}}(q_{k}, q_{k+1})) $ $+ \beta\cdot{\stackrel[k=K]{K+l-1}{\sum}D_{V}(L_{X}(x_{\mathcal{R},k+1}), \ensuremath{\mathcal{X}}(q_{k}, q_{k+1}))}$}. 
\end{lem}

Lemma~\ref{lem:discrete_violation} states that we can measure the total violation cost of an infinite continuous path using a finite discrete form.
The structure of $\mathcal{R}_{\mathcal{T}}$ can enable a modular structure of distributedly solving each $\mathcal{R}_{i}$ with shorter horizons. Note that, the above process generates the asymptotically minimally-violating plan that is probabilistically complete.

\begin{figure}[!t]
\centering
\centerline{\includegraphics[width=\columnwidth]{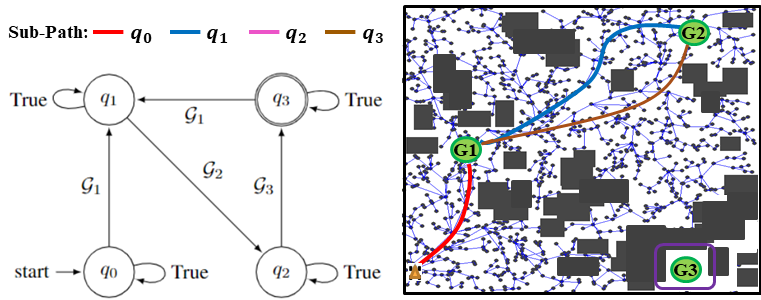}}
\caption{(Left) NBA $\mathcal{B}$ of the LTL formula $\phi_{g_{1}}=\Always\Eventually\mathcal{G}_{1}\land\Always\Eventually\mathcal{G}_{2}\land\Always\Eventually\mathcal{G}_{3}$ for $\phi=\Always\lnot\mathcal{O}\land\phi_{g_{1}}$; (Right)  Decomposed reach-avoid paths (sub-tasks) that minimally violates $\phi$, where $\mathcal{G}_{3}$ is not reachable.
}
\label{fig:Infeasible_example}
\end{figure}

\begin{example}
\label{example2}
Fig.~\ref{fig:Infeasible_example} shows an example of the optimal decomposition for the infeasible case. The LTL task is $\phi=\Always\lnot\mathcal{O}\land\phi_{g_{1}}=\Always\Eventually\mathcal{G}_{1}\land\Always\Eventually\mathcal{G}_{2}\land\Always\Eventually\mathcal{G}_{3}$ over infinite horizons, whereas the goal regions $\mathcal{G}_{3}$ is not accessible.
The resulting NBA and decomposed trajectories of TL-RRT* minimally-violating $\phi$ are shown in Fig.~\ref{fig:Infeasible_example} (left) and (right), respectively, where decomposed reach-avoid paths (sub-tasks) are expressed as $\mathcal{R}_{\mathcal{T}}=\mathcal{R}_{red}(\mathcal{R}_{blue}\mathcal{R}_{brown})^{\omega}$.  
\end{example}

\begin{figure*}
\vskip 0.1in
\begin{center}
\centerline{\includegraphics[scale=0.5 ]{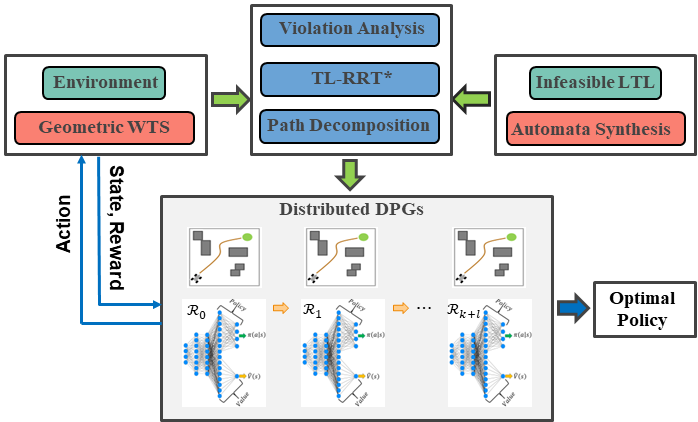}}
\caption{The overall diagram of the framework includes geometric minimally-violating planning for RL guidance, decomposition, and distributed training.}

\label{fig:diagram}
\end{center}
\end{figure*}

\subsection{Reward Design and Distributed DRL\label{subsec:RL}}

\textbf{Reward design:\space} According to the optimal compositional plan $\mathcal{R}_{\mathcal{T}}$, this section designs a reward function for each $\mathcal{R}_{i}$ that is dense in the configuration space to overcome the challenge of the complex environment. We can train every reach-avoid $\mathcal{R}_{i}$ navigation control in the same way distributively. Even if we can straightforwardly use the geometric sub-path for reward design, there exist three issues: (a) without considering actual dynamical systems, controllers cannot strictly track the geometric path; (b) the robot may linger around previous waypoints resulting in sub-optimal solutions;  (c) the waypoints might overlap with obstacles. We address them in the following descriptions, and the intuition of the reward design is shown in Fig.~\ref{fig:running_example} (b), where the robot just needs to move towards the goal and is not required to strictly follow the geometric path.

\textbf{(a):\space} Each geometric path of $\mathcal{R}_{i}$ is a sequence of waypoints  $\tau^{*}_{\mathcal{T}, i}= x_{i,0}x_{i,1}\ldots x_{i,N}$. We construct an r-norm ball for each state $x_{i,j}\in\tau^{*}_{\mathcal{T}, i}, \forall j=0,1\ldots,N$, as $Ball_{r}(x_{i,j})=\left\{x\in X\mid dist(x,x_{i,j})\leq r\right\}$. These balls allow the robot to pass them and move towards the goal as guidance, which we use to design a reward function. The consecutive balls can be regarded as a tube which is popular in the robust control community~\cite{mayne2011tube} to handle path tracking infeasibility.

\textbf{(b):\space} The intuitive approach is to track the minimum distance of the visited balls to the destination so far and utilize it as a constraint in the reward design. Since TL-RRT* is an extension of RRT*~\cite{karaman2011sampling}, it also provides the optimal distance function in the tree from each state to the global final destination denoted as $Cost(x)$. For a sub-path $\tau^{*}_{\mathcal{T}, i}= x_{i,0}x_{i,1}\ldots x_{i,N}$, we compute the distance from each state $x_{i,j}\in\tau^{*}_{\mathcal{T}, i}$ to the destination $x_{i,N}$ as $Dist(x_{i,j}) = Cost(x_{i,j})-Cost(x_{i,N})$. We return the necessary distance information as:
{\small
\begin{equation}
D(x)=
\begin{cases}
Dist(x_{i}|\boldsymbol{x}^*) & \text{if } x\in Ball_{r}(x_{i}|\boldsymbol{x}^*)\\
\infty & \text{otherwise}
\end{cases}
\label{eq:progression}
\end{equation}
}

During each episode of training, a state-action sequence $s_{0}a_{0}s_{1}a_{i}\ldots s_{t}$ up to current time $t$ is split into the state and action sequences  $\boldsymbol{s}_{t}=s_{0}s_{1}\ldots s_{t}$ and $\boldsymbol{a}_{t}=a_{0}a_{1}\ldots a_{t-1}$, respectively. We develop a progression function $D_{min}$ to identify whether the next state is getting closer to the goal region as
$D_{min}(\boldsymbol{s}_{t})=\underset{s\in \boldsymbol{s}_{t}}{\min}\left\{D(Proj(s))\right\}$.

The learning objective is to push $D_{min}$ to decrease at each time-step. We design a constrained reward for each time $t$ of an episode as $R(s^{\times}_{t})=r_{+}>0 \text{ if } D(Proj(s_{t}))< D_{min}(\boldsymbol{s}_{t-1})$. Moreover, such a design also alleviates the first issue (a), since the RL agent is not required to visit every r-norm ball. 

\textbf{(c):\space} We design the reward in different priorities such that if the RL-agent collides with obstacles, it will always return the negative reward and ignore other conditions. In summary, our reward design is the following with decreasing priorities:
{\small
\begin{equation}
\arraycolsep=1.4pt
\begin{array}{lr}
R(s_{t})= r_{-}<0, & \text{if } Proj(s_{t})\in X_{\mathcal{O}},\\
R(s_{t})= r_{++}>0, & \text{if } D(Proj(s_{t}))=0,
\\
R(s_{t})= r_{+}>0, & \text{if } D(Proj(s_{t}))< D_{min}(\boldsymbol{s}_{t-1}),\\
R(s_{t})= 0, & \text{otherwise,}
\end{array}.\label{eq:reward_function}
\end{equation}
}

We evaluate the performance of the reward design~\eqref{eq:reward_function} as:

\begin{thm}
\label{thm:RRT*}
If there exists at least one policy $\pi^{*}_{\mathcal{R},i}$ satisfying the reach-avoid task $\mathcal{R}_{i}$, by selecting $r_{++}$ to be sufficiently larger than $r_{+}$, i.e., $r_{++} \gg r_{+}$, any algorithm that optimizes the expected return $J(\pi)$ is guaranteed to find such an optimal policy $\pi^{*}_{\mathcal{R},i}$.
\end{thm}

\begin{proof}
Theorem~\ref{thm:RRT*} can be proved by contradiction.
Suppose we have a policy $\bar{\pi}_{\mathcal{R},i}$ that is optimal and does not accomplish reach-avoid task $\mathcal{R}_{i}$, which means the robot derived by $\bar{\pi}_{\mathcal{R},i}$ will not reach the goal station. 
Recall that we have $N$ r-norm balls (waypoints), and the best case of $\bar{\pi}_{\mathcal{R},i}$ is to consecutively pass all these balls sequentially without reaching the destination. We obtain the upper-bound 
{\small
\begin{equation}
    J(\bar{\pi}_{\mathcal{R},i})< r_{+}\cdot\frac{1-\gamma^{N}}{1-\gamma}\label{eq:upper}
\end{equation}
}

Per assumption 1, we  can find another policy $\pi^{*}_{\mathcal{R},i}$ that reaches and stays at the destination repetitively receiving reward $r_{++}$.
The worst case of $\pi^{*}_{\mathcal{R},i}$ is to pass no r-norm waypoints. In this case, we obtain the lower-bound 
{\small
\begin{equation}
    J(\pi^{*}_{\mathcal{R},i})\geq \underline{M}r_{++}\cdot\frac{1}{1-\gamma}\label{eq:lower},
\end{equation}}%
where $\underline{M}=\gamma^{\bar{n}}$, and $\bar{n}$ is maximum number of steps reaching the goal station. Consequently, for~\eqref{eq:upper} and~\eqref{eq:lower}, if we select
{\small
\begin{equation}
    r_{++}\geq \frac{r_{+}\cdot(1-\gamma^{N})}{\underline{M}},
\end{equation}}%
we guarantee that $J(\pi^{*}_{\mathcal{R},i})>J(\bar{\pi}_{\mathcal{R},i})$, which
contradicts the fact that $\bar{\pi}_{\mathcal{R},i}$ is an optimal policy.
\end{proof}

Note that even though training deep neural networks to optimize policies is a nonlinear regression process, we can use neural network verification techniques, e.g., \cite{everett2021neural} to verify and improve the learned policy $\pi^{*}_{\mathcal{R},i}$ for a reach-avoid task.

\begin{rem}
Employing $D_{min} (\boldsymbol{s}_{t})$ to design a reward function results in a non-Markovian property, since it relies on the past history $\boldsymbol{s}_{t}$, whereas the reward function~\eqref{eq:reward_function} only depends on current state. 
To address this, we can leverage the structure of the product MDP~\cite{baier2008}, and augment the current state $s_{t}$ with the index of the closest visited waypoint, i.e., $\underset{x\in \boldsymbol{x}_{t}}{\arg\min}\left\{D(x)\right\}$, to construct a product state to recover the Markovian property. 
\end{rem}

\textbf{Distributed DRL:\space}
For each $\mathcal{R}_{i}$, we can apply the above approach to an existing off-the-shelf DRL algorithms, e.g., SAC~\cite{ haarnoja2018soft}, to train the optimal policy $\pi^{*}_{\mathcal{R}, i}$. The process can be repeated for all $\mathcal{R}_{i}, \forall i\in\left\{0,1,\ldots K+l\right\}$ of  $\mathcal{R}_{\mathcal{T}}=(\mathcal{R}_{0}\mathcal{R}_{1}\ldots \mathcal{R}_{K})(\mathcal{R}_{K+1}\ldots 
\mathcal{R}_{K+l})^{\omega}$ in a distributed manner.
We then concatenate them to obtain the globally optimal policy as 
\begin{equation}
    \pi^{*}_{\phi}=(\pi^{*}_{\mathcal{R},0}\pi^{*}_{\mathcal{R},i}\ldots \pi^{*}_{\mathcal{R},K})(\pi^{*}_{\mathcal{R},K+1}\ldots \pi^{*}_{\mathcal{R},K+l})^{\omega}, \label{eq:policy}
\end{equation}
where $\pi^{*}_{\phi}$ is the global policy satisfying LTL task $\phi$ and containing a set of neural network parameters. Note that if we train each reach-avoid task $\mathcal{R}_{i}$ individually, we need to randomize the initial velocity and acceleration conditions of state $S$, since the ending condition of completing a policy $\pi^{*}_{\mathcal{R}, i}$ results in different starting states for the next policy $\pi^{*}_{\mathcal{R},i+1}$. This enables a smooth concatenation instead of re-training the whole global policy $\pi^{*}_{\theta}$. 

The overall framework summary can be found in Fig.~\ref{fig:diagram}. It mainly includes three stages i.e., minimally-violating planning for the infeasible LTL specification as RL guidance, path decomposition, and distributed training.

\begin{thm}
\label{thm:minimum-violation}
If for all $i\in\left\{0,1,\ldots K+l\right\}$, $\pi^{*}_{\mathcal{R}, i}$ satisfies its corresponding reach-avoid task $\mathcal{R}_{i}$ of $\mathcal{R}_{\mathcal{T}}$ in~\eqref{eq:decomposition} under any initial states (velocities, accelerations), then their concatenation $\pi^{*}_{\phi}$~in~\eqref{eq:policy} satisfies the LTL task $\phi$ with minimum-violation guarantees.
\end{thm}

\begin{proof}
The intuition is to prove that every admissible path  path $\boldsymbol{s}_{\infty}^{\pi^{*}_{\phi}}\in\boldsymbol{S}_{\infty}^{\pi^{*}_{\phi}}$ under the optimal policy $\pi^{*}_{\phi}$ satisfies the LTL specification $\phi$ with minimum total violation cost.
The proof is based on the decomposition property in Section~\ref{subsec:minimum}.

First, the geometric projection of every admissible path $\boldsymbol{s}_{\infty}^{\pi^{*}_{\phi}}\in\boldsymbol{S}_{\infty}^{\pi^{*}_{\phi}}$ satisfies the lasso-type reach-avoid $\mathcal{R}_{\mathcal{T}}$ in~\eqref{eq:decomposition}, which is associated with a lasso-type automaton path $\tau^{*}_{Q}$ as~\eqref{eq:automaton}. From Lemma~\ref{lem:discrete_violation} and the optimality of TL-RRT*~\cite{luo2021abstraction}, $\tau^{*}_{Q}$ satisfies the acceptance condition of the NBA of LTL formula $\phi$ and has the minimum total violation cost. 
\end{proof}

\begin{rem}
   Due to the conclusion of minimally-violating the LTL task from Theorem~\ref{thm:minimum-violation}, our framework generalizes to both feasible and infeasible cases for a given LTL task.
\end{rem}

\section{EXPERIMENTAL RESULTS}

We implement the framework on two different nonlinear dynamical systems for various LTL tasks over both finite and infinite horizons. The algorithm test focuses on infeasible cases and complex cluttered environments, where dense cluttered obstacles are randomly sampled. We also show that the framework can complete feasible LTL tasks exactly.
We apply SAC~\cite{haarnoja2018soft} as a state-of-art DRL algorithm for all baselines.

\textbf{Baseline Approaches:} 
From the learning perspective, we refer to our distributed framework as "RRT*" or "D-RRT*", and compare it against three baselines: (i) The relaxed TL-based multi-objective rewards in~\cite{Cai2020, Cai2021modular} referred to as "TL",  for the single LTL task;
(ii) For the goal-reaching task $\phi$, the baseline referred to as "NED" designs the reward based on the negative Euclidean distance between the robot and destination; (iii) For a complex LTL task, instead of decomposition, this baseline directly applies the reward scheme~\eqref{eq:reward_function} for the global trajectory $\tau^{*}_{F}=\tau^{*}_{pre}[\tau^{*}_{suf}]^{\omega}$ referred to as "G-RRT*".

From the perspective of infeasible LTL tasks, we compare with the work~\cite{hasanbeig2019reinforcement} of visiting as many accepting sets of LDGBA~\cite{Sickert2016} as possible, and we empirically show the improvement over our prior work~\cite{cai2022overcoming} that assumes the feasible cases.

\begin{figure}
\centering
\centerline{\includegraphics[width=\columnwidth]{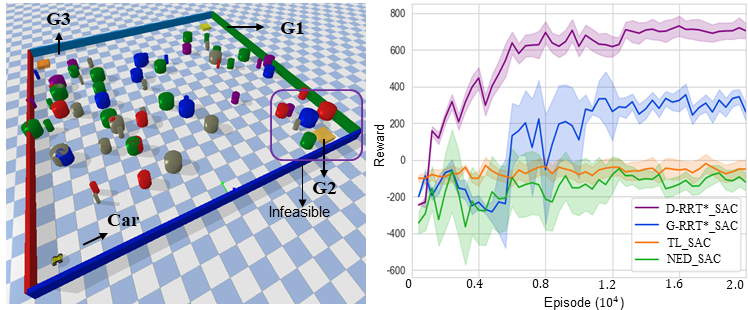}}
\caption{Baselines comparison of the infeasible task $\phi_{1,inf}$ for a $2$D Dubins car in the Pybullet environment with cluttered obstacles.}
\label{fig:vehicle}
\end{figure}

\textbf{Autonomous Vehicle\text{ }} We first implement the Dubins car model in the Pybullet\footnote{\url{https://pybullet.org/wordpress/}} physics engine shown in Fig.~\ref{fig:vehicle}. We consider the  surveillance LTL task $\phi_{1,inf}=\Always\lnot\mathcal{O}\land\Always\Eventually(\mathcal{G}_{1}\land\Eventually(\mathcal{G}_{2}\land\Eventually(\mathcal{G}_{3}\land\Eventually(\mathcal{G}_{init})))$ that requires sequentially visiting regions labeled as $\mathcal{G}_{1}, \mathcal{G}_{2}, \mathcal{G}_{3}$, and the robot's initial position infinitely often. Its finite-horizon version can be expressed as  
 $\phi_{1,fin}=\Always\lnot\mathcal{O}\land\Eventually(\mathcal{G}_{1}\land\Eventually(\mathcal{G}_{2}\land\Eventually(\mathcal{G}_{3}\land\Eventually(\mathcal{G}_{init})))$ by removing the \emph{always} operator $\Always$. Both tasks are infeasible since the region $\mathcal{G}_{2}$ is surrounded by obstacles.
Fig.~\ref{fig:vehicle} shows the learning curves of task $\phi_{1,inf}$ compared with different baselines, and $\phi_{1,fin}$ has the same comparison results. We can observe that our framework can provide better performance than other baselines under the challenge of complex cluttered environments.

\begin{figure}
\centering
\centerline{\includegraphics[width=\columnwidth]{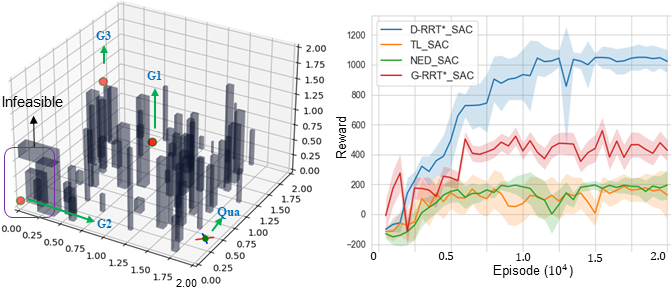}}
\caption{ Baselines comparison of the infeasible task $\phi_{2,inf}$ in the $3$D complex cluttered environment.}
\label{fig:drone}
\end{figure}

\textbf{Quadrotor Model\text{ }} We test our algorithms in a $3$D environment with Quadrotor\footnote{\url{https://github.com/Bharath2/Quadrotor-Simulation/tree/main/PathPlanning}} dynamics shown in Fig.~\ref{fig:drone}. It demonstrates that our model-free framework is capable of handling complex cluttered environments with high dimensions.
The LTL specification is given as $\phi_{2,inf}=\Always\lnot\mathcal{O}\land\Always\Eventually\mathcal{G}_{1}\land\Always\Eventually\mathcal{G}_{2}\land\Always\Eventually\mathcal{G}_{3}$, which requires the quadrotor to navigate regions of $\mathcal{G}_{1}, \mathcal{G}_{2}, \mathcal{G}_{3}$ infinitely often without specific orders. Its finite version can also be defined as $\phi_{2,fin}=\Always\lnot\mathcal{O}\land\Eventually\mathcal{G}_{1}\land\Eventually\mathcal{G}_{2}\land\Eventually\mathcal{G}_{3}$. These two task are also infeasible and, 
their comparisons of learning performances are the same. Due to page limitation, we only show the learning results for the task $\phi_{2,inf}$ in Fig~\ref{fig:drone}.

\textbf{Feasibility Generalization\text{ }} According to theorem~\ref{thm:minimum-violation}, the advantage of minimally violating a given task allows generalization of feasible cases. To show the metric, we remove the  highlighted obstacles for the region $\mathcal{G}_{2}$ in both environments, and  implement our algorithms. The feasible experimental results shown in the video demo demonstrate that our framework can satisfy feasible tasks exactly.

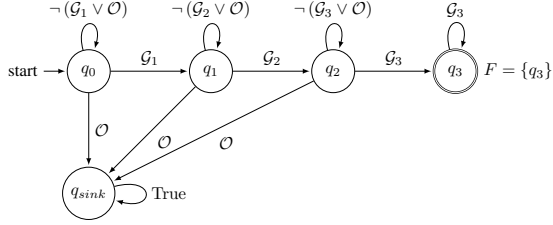
\begin{figure}[!t]\centering
	{{
	    \scalebox{.65}{
			\begin{tikzpicture}[shorten >=1pt,node distance=2.5cm,on grid,auto] 
			\node[state,initial] (q_0)   {$q_0$}; 
			\node[state] (q_1) [right=of q_0]  {$q_1$};
			\node[state] (q_2) [right=of q_1]  {$q_2$};
			\node[state,accepting, label=right:$\textcolor{black}{F=\{q_3\}}$] (q_3) [right=of q_2]  {$q_3$};
			\node[state] (q_4) [below=of q_0]  {$q_{sink}$};
			\path[->] 
			(q_0) edge [bend left=0] node {$\mathcal{G}_{1}$} (q_1)
			(q_0) edge [loop above] node {$\lnot\left(\mathcal{G}_{1}\vee \mathcal{O}\right)$} (q_0)
			(q_0) edge [bend left=0] node {$\mathcal{O}$} (q_4)
			(q_1) edge [bend left=0] node {$\mathcal{G}_{2}$} (q_2)
			(q_1) edge [loop above] node {$\lnot\left(\mathcal{G}_{2}\vee \mathcal{O}\right)$} (q_1)
			(q_1) edge [bend left=0] node {$\mathcal{O}$} (q_4)
			(q_2) edge [bend left=0] node {$\mathcal{G}_{3}$} (q_3)
			(q_2) edge [loop above] node {$\lnot\left(\mathcal{G}_{3}\vee \mathcal{O}\right)$} (q_2)
			(q_2) edge [bend left=0] node {$\mathcal{O}$} (q_4)
			(q_3) edge [loop above] node {$\mathcal{G}_{3}$} (q_3)
			(q_4) edge [loop right] node {$\operatorname{True}$} (q_4);
			\end{tikzpicture}
			}
			}}
    \caption{\label{fig:automaton} LDGBA of the LTL formula $\phi_{1,fin}=\Always\lnot\mathcal{O}\land\Eventually(\mathcal{G}_{1}\land\Eventually(\mathcal{G}_{2}\land\Eventually(\mathcal{G}_{3}))$ with only one accepting set $F=\{q_3\}$.}
    \vspace{-0.3cm}
\end{figure}

\textbf{5.5 Infeasibility Analysis\text{ }}
For each environment and dynamical system, we increase the task complexity by randomly sampling $12$ obstacle-free goal regions in both environments and set the specifications as
$\phi_{3,fin}=\Always\lnot\mathcal{O}\land(\left(\Eventually\mathtt{\mathcal{G}_{1}}\land\Eventually\mathtt{\left(\mathcal{G}_{2}\land\Eventually\mathtt{\ldots\land \Eventually\mathcal{G}_{12}}\right)}\right)$, and
$\phi_{3,inf}=\Always\lnot\mathcal{O}\land\Always\Eventually\mathcal{G}_{1}\land\Always\Eventually\mathcal{G}_{2}\ldots\land\Always\Eventually\mathcal{G}_{12}$. We repeat the random sampling for $50$ trials and record the success rates of generating valid plans for two environments. The results for these task $\phi_{3,fin}, \phi_{3,inf}$
compared with our prior work~\cite{cai2022overcoming} referred to as "feasible" are shown in Fig.~\ref{fig:analysis} (a).
Due to the random process, the target regions may be surrounded by obstacles, and the baseline returns no solutions, whereas our framework always finds minimally-violating trajectories.

\begin{figure}
\centering
\centerline{\includegraphics[width=\columnwidth]{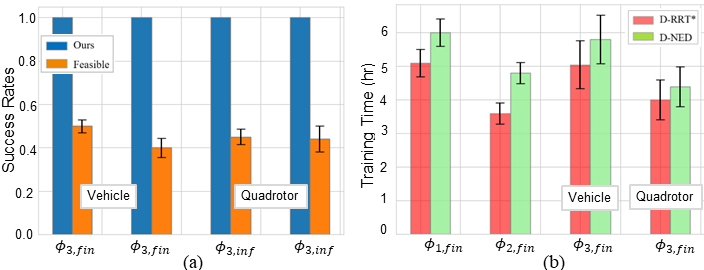}}
\caption{ (a) Success rates of random goal regions compared with~\cite{cai2022overcoming}. (b) Training time for each reach-aviod task compared with distributed baseline "NED" referred as "D-NED".}
\label{fig:analysis}
\end{figure}

From the aspect of automaton-based rewards, the work~\cite{Cai2020} directly applies a linear combination of violation cost and automaton acceptance reward to formulate a MORL process, 
with no guarantees on satisfaction and violation. Our proposed framework leveraging planning methods as guidance improves the performance. This seen in the learning comparisons in Fig.~\ref{fig:vehicle} and Fig.~\ref{fig:drone}. 
Moreover, we also compare the work~\cite{hasanbeig2019reinforcement} that addresses infeasible cases for discrete MDP model. It finds the policy that satisfies the given LTL tasks as much as possible by intersecting with the maximum number of accepting sets of an LDGBA~\cite{Sickert2016}. However,
the LDGBAs of some LTL formulas only have one accepting set. As an example, Fig.~\ref{fig:automaton} shows the LDGBA of 
$\phi_{1,fin}=\Always\lnot\mathcal{O}\land\Eventually(\mathcal{G}_{1}\land\Eventually(\mathcal{G}_{2}\land\Eventually(\mathcal{G}_{3}))$ for an autonomous vehicle that has only one accepting set $\{q_3\}$. The same case also applies to the LTL formula $\phi_{1,inf}, \phi_{3,fin}$. In such cases, the work~\cite{hasanbeig2019reinforcement} returns no solution since the only accepting set is not accessible.

\textbf{5.6 Training efficiency\text{ }} We record the training time for each reach-avoid sub-task for the finite-horizon LTL tasks i.e., $\phi_{1,fin}, \phi_{2,fin}, \phi_{3,fin}$ in two environments. We compared with the baseline "NED" that distributedly solves each reach-avoid sub-task using distance based rewards. The results in Fig.~\ref{fig:analysis} (b) shows that our framework is more efficient.

\section{CONCLUSION}

In this article, we propose a model-free framework using DRL to learn continuous controllers for completing complex LTL navigation tasks. We overcome the challenge of infeasible LTL specifications in complex cluttered environments. To minimally violate the LTL task, we apply path planning methods as guidance of DRL that can overcome challenges of cluttered environments and guarantee minimal total violation cost, which can decompose the task into sub-reach-avoid missions.
 Based on the advantage of minimal violation, our work generalizes both feasible and infeasible cases. In the future, we will consider safety-critical exploration during learning, and shrinking the gap of sim-to-real to deploy neural network controllers in real-world robots.

\bibliographystyle{IEEEtran}
\bibliography{reference}

\begin{thebibliography}{10}
\providecommand{\url}[1]{#1}
\csname url@samestyle\endcsname
\providecommand{\newblock}{\relax}
\providecommand{\bibinfo}[2]{#2}
\providecommand{\BIBentrySTDinterwordspacing}{\spaceskip=0pt\relax}
\providecommand{\BIBentryALTinterwordstretchfactor}{4}
\providecommand{\BIBentryALTinterwordspacing}{\spaceskip=\fontdimen2\font plus
\BIBentryALTinterwordstretchfactor\fontdimen3\font minus
  \fontdimen4\font\relax}
\providecommand{\BIBforeignlanguage}[2]{{%
\expandafter\ifx\csname l@#1\endcsname\relax
\typeout{** WARNING: IEEEtran.bst: No hyphenation pattern has been}%
\typeout{** loaded for the language `#1'. Using the pattern for}%
\typeout{** the default language instead.}%
\else
\language=\csname l@#1\endcsname
\fi
#2}}
\providecommand{\BIBdecl}{\relax}
\BIBdecl

\bibitem{sutton2018reinforcement}
R.~S. Sutton and A.~G. Barto, \emph{Reinforcement learning: An
  introduction}.\hskip 1em plus 0.5em minus 0.4em\relax MIT press, 2018.

\bibitem{haarnoja2018soft}
T.~Haarnoja, A.~Zhou, P.~Abbeel, and S.~Levine, ``Soft actor-critic: Off-policy
  maximum entropy deep reinforcement learning with a stochastic actor,'' in
  \emph{Intl Conference on Machine Learning}, 2018, pp. 1861--1870.

\bibitem{long2018towards}
P.~Long, T.~Fan, X.~Liao, W.~Liu, H.~Zhang, and J.~Pan, ``Towards optimally
  decentralized multi-robot collision avoidance via deep reinforcement
  learning,'' in \emph{IEEE International Conference on Robotics and Automation
  (ICRA)}, 2018, pp. 6252--6259.

\bibitem{chiang2019learning}
H.-T.~L. Chiang, A.~Faust, M.~Fiser, and A.~Francis, ``Learning navigation
  behaviors end-to-end with autorl,'' \emph{IEEE Robotics and Automation
  Letters}, vol.~4, no.~2, pp. 2007--2014, 2019.

\bibitem{sadigh2014learning}
D.~Sadigh, E.~S. Kim, S.~Coogan, S.~S. Sastry, and S.~A. Seshia, ``A learning
  based approach to control synthesis of markov decision processes for linear
  temporal logic specifications,'' in \emph{IEEE Conference on Decision and
  Control}, 2014, pp. 1091--1096.

\bibitem{kantaros2022accelerated}
Y.~Kantaros, ``Accelerated reinforcement learning for temporal logic control
  objectives,'' \emph{Intl Conf on Intelligent Robots and Systems}, 2022.

\bibitem{xu2020joint}
Z.~Xu, I.~Gavran, Y.~Ahmad, R.~Majumdar, D.~Neider, U.~Topcu, and B.~Wu,
  ``Joint inference of reward machines and policies for reinforcement
  learning,'' in \emph{Proceedings of the International Conference on Automated
  Planning and Scheduling}, vol.~30, 2020, pp. 590--598.

\bibitem{cai2021reinforcement}
M.~Cai, S.~Xiao, B.~Li, Z.~Li, and Z.~Kan, ``Reinforcement learning based
  temporal logic control with maximum probabilistic satisfaction,'' in
  \emph{IEEE Intl Conf on Robotics and Automation}, 2021, pp. 806--812.

\bibitem{li2019formal}
X.~Li, Z.~Serlin, G.~Yang, and C.~Belta, ``A formal methods approach to
  interpretable reinforcement learning for robotic planning,'' \emph{Science
  Robotics}, vol.~4, no.~37, p. eaay6276, 2019.

\bibitem{icarte2022reward}
R.~T. Icarte, T.~Q. Klassen, R.~Valenzano, and S.~A. McIlraith, ``Reward
  machines: Exploiting reward function structure in reinforcement learning,''
  \emph{Journal of Artificial Intelligence Research}, vol.~73, pp. 173--208,
  2022.

\bibitem{Cai2021modular}
M.~Cai, M.~Hasanbeig, S.~Xiao, A.~Abate, and Z.~Kan, ``Modular deep
  reinforcement learning for continuous motion planning with temporal logic,''
  \emph{IEEE Robotics and Automation Letters}, vol.~6, no.~4, pp. 7973--7980,
  2021.

\bibitem{cai2021safe}
M.~Cai and C.-I. Vasile, ``Safe-critical modular deep reinforcement learning
  with temporal logic through gaussian processes and control barrier
  functions,'' \emph{arXiv preprint arXiv:2109.02791}, 2021.

\bibitem{cai2022overcoming}
M.~Cai, E.~Aasi, C.~Belta, and C.-I. Vasile, ``Overcoming exploration: Deep
  reinforcement learning for continuous control in cluttered environments from
  temporal logic specifications,'' \emph{IEEE Robotics and Automation Letters},
  2023.

\bibitem{explorationref}
H.~Tang, R.~Houthooft, D.~Foote, A.~Stooke, X.~Chen, Y.~Duan, J.~Schulman,
  F.~De~Turck, and P.~Abbeel, ``Exploration: A study of count-based exploration
  for deep reinforcement learning,'' in \emph{Neural Information Processing
  Systems}, 2017, pp. 2750--2759.

\bibitem{hrl-survey}
O.~Nachum, H.~Tang, X.~Lu, S.~Gu, H.~Lee, and S.~Levine, ``Why does hierarchy
  (sometimes) work so well in reinforcement learning?'' 2019.

\bibitem{raman2014model}
V.~Raman, A.~Donz{\'e}, M.~Maasoumy, R.~M. Murray, A.~Sangiovanni-Vincentelli,
  and S.~A. Seshia, ``Model predictive control with signal temporal logic
  specifications,'' in \emph{IEEE Conference on Decision and Control}, 2014,
  pp. 81--87.

\bibitem{rodionova2022combined}
A.~Rodionova, L.~Lindemann, M.~Morari, and G.~J. Pappas, ``Combined left and
  right temporal robustness for control under stl specifications,'' \emph{IEEE
  Control Systems Letters}, 2022.

\bibitem{Tumova2013}
J.~Tumova, G.~C. Hall, S.~Karaman, E.~Frazzoli, and D.~Rus, ``Least-violating
  control strategy synthesis with safety rules,'' in \emph{Proc. Int. Conf.
  Hybrid syst., Comput. Control}, 2013, pp. 1--10.

\bibitem{Lahijanian2016}
M.~Lahijanian, M.~R. Maly, D.~Fried, L.~E. Kavraki, H.~Kress-Gazit, and M.~Y.
  Vardi, ``Iterative temporal planning in uncertain environments with partial
  satisfaction guarantees,'' \emph{IEEE Transactions on Robotics}, vol.~32,
  no.~3, pp. 583--599, 2016.

\bibitem{vasile2017minimum}
C.-I. Vasile, J.~Tumova, S.~Karaman, C.~Belta, and D.~Rus, ``{Minimum-violation
  scLTL motion planning for mobility-on-demand},'' in \emph{IEEE Intl
  Conference on Robotics and Automation (ICRA)}, 2017, pp. 1481--1488.

\bibitem{niu2018minimum}
L.~Niu, J.~Fu, and A.~Clark, ``Minimum violation control synthesis on
  cyber-physical systems under attacks,'' in \emph{2018 IEEE Conference on
  Decision and Control (CDC)}.\hskip 1em plus 0.5em minus 0.4em\relax IEEE,
  2018, pp. 262--269.

\bibitem{Lacerda2019}
B.~Lacerda, F.~Faruq, D.~Parker, and N.~Hawes, ``Probabilistic planning with
  formal performance guarantees for mobile service robots,'' \emph{Intl Journal
  of Robotics Research}, vol.~38, no.~9, pp. 1098--1123, 2019.

\bibitem{niu2020optimal}
L.~Niu, J.~Fu, and A.~Clark, ``Optimal minimum violation control synthesis of
  cyber-physical systems under attacks,'' \emph{IEEE Transactions on Automatic
  Control}, vol.~66, no.~3, pp. 995--1008, 2020.

\bibitem{rahmani2020you}
H.~Rahmani and J.~M. O'Kane, ``What to do when you can't do it all: Temporal
  logic planning with soft temporal logic constraints,'' in \emph{Intl.
  Conference on Intelligent Robots and Systems}, 2020, pp. 6619--6626.

\bibitem{wongpiromsarn2021minimum}
T.~Wongpiromsarn, K.~Slutsky, E.~Frazzoli, and U.~Topcu, ``{Minimum-violation
  planning for autonomous systems: Theoretical and practical considerations},''
  in \emph{American Control Conference}, 2021, pp. 4866--4872.

\bibitem{KaKaVa-IROS-2021}
D.~Kamale, E.~Karyofylli, and C.~I. Vasile, ``{Automata-based Optimal Planning
  with Relaxed Specifications},'' in \emph{IEEE/RSJ Intl Conference on
  Intelligent Robots and Systems (IROS)}, 2021, pp. 6525--6530.

\bibitem{kim2015minimal}
K.~Kim, G.~Fainekos, and S.~Sankaranarayanan, ``On the minimal revision problem
  of specification automata,'' \emph{The International Journal of Robotics
  Research}, vol.~34, no.~12, pp. 1515--1535, 2015.

\bibitem{Guo2015}
M.~Guo and D.~V. Dimarogonas, ``Multi-agent plan reconfiguration under local
  {LTL} specifications,'' \emph{International Journal of Robotics Research},
  vol.~34, no.~2, pp. 218--235, 2015.

\bibitem{cai2020receding}
M.~Cai, H.~Peng, Z.~Li, H.~Gao, and Z.~Kan, ``Receding horizon control-based
  motion planning with partially infeasible ltl constraints,'' \emph{IEEE
  Control Systems Letters}, vol.~5, no.~4, pp. 1279--1284, 2020.

\bibitem{cai2021optimal}
M.~Cai, S.~Xiao, Z.~Li, and Z.~Kan, ``Optimal probabilistic motion planning
  with potential infeasible ltl constraints,'' \emph{IEEE Transactions on
  Automatic Control}, 2021.

\bibitem{hasanbeig2019reinforcement}
M.~Hasanbeig, Y.~Kantaros, A.~Abate, D.~Kroening, G.~J. Pappas, and I.~Lee,
  ``Reinforcement learning for temporal logic control synthesis with
  probabilistic satisfaction guarantees,'' in \emph{IEEE Conference on Decision
  and Control (CDC)}, 2019, pp. 5338--5343.

\bibitem{Cai2020}
M.~Cai, H.~Peng, Z.~Li, and Z.~Kan, ``Learning-based probabilistic ltl motion
  planning with environment and motion uncertainties,'' \emph{IEEE Tran on
  Automatic Control}, vol.~66, no.~5, pp. 2386--2392, 2021.

\bibitem{Sickert2016}
S.~Sickert, J.~Esparza, S.~Jaax, and J.~K{\v{r}}et{\'\i}nsk{\`y},
  ``Limit-deterministic {B}{\"u}chi automata for linear temporal logic,'' in
  \emph{Int. Conf. Comput. Aided Verif.}\hskip 1em plus 0.5em minus 0.4em\relax
  Springer, 2016, pp. 312--332.

\bibitem{van2014multi}
K.~Van~Moffaert and A.~Now{\'e}, ``Multi-objective reinforcement learning using
  sets of pareto dominating policies,'' \emph{The Journal of Machine Learning
  Research}, vol.~15, no.~1, pp. 3483--3512, 2014.

\bibitem{qin2003survey}
S.~J. Qin and T.~A. Badgwell, ``A survey of industrial model predictive control
  technology,'' \emph{Control engineering practice}, vol.~11, no.~7, pp.
  733--764, 2003.

\bibitem{manchester2017control}
I.~R. Manchester and J.-J.~E. Slotine, ``Control contraction metrics: Convex
  and intrinsic criteria for nonlinear feedback design,'' \emph{IEEE Tran on
  Automatic Control}, vol.~62, no.~6, pp. 3046--3053, 2017.

\bibitem{chen2021fastrack}
M.~Chen, S.~L. Herbert, H.~Hu, Y.~Pu, J.~F. Fisac, S.~Bansal, S.~Han, and C.~J.
  Tomlin, ``Fastrack: a modular framework for real-time motion planning and
  guaranteed safe tracking,'' \emph{IEEE Transactions on Automatic Control},
  vol.~66, no.~12, pp. 5861--5876, 2021.

\bibitem{luo2021abstraction}
X.~Luo, Y.~Kantaros, and M.~M. Zavlanos, ``An abstraction-free method for
  multirobot temporal logic optimal control synthesis,'' \emph{IEEE
  Transactions on Robotics}, 2021.

\bibitem{VaLiBe-IJRR-2020}
C.~I. Vasile, X.~Li, and C.~Belta, ``{Reactive Sampling-Based Path Planning
  with Temporal Logic Specifications},'' \emph{International Journal of
  Robotics Research}, vol.~39, no.~8, pp. 1002--1028, June 2020.

\bibitem{Earl1955}
E.~A. Coddington and N.~Levinson, ``Theory of ordinary differential
  equations.'' \emph{McGraw-Hill}, pp. 1--42, 1955.

\bibitem{baier2008}
C.~Baier and J.-P. Katoen, \emph{Principles of model checking}.\hskip 1em plus
  0.5em minus 0.4em\relax MIT press, 2008.

\bibitem{kloetzer2008fully}
M.~Kloetzer and C.~Belta, ``A fully automated framework for control of linear
  systems from temporal logic specifications,'' \emph{IEEE Transactions on
  Automatic Control}, vol.~53, no.~1, pp. 287--297, 2008.

\bibitem{vardi1986automata}
M.~Y. Vardi and P.~Wolper, ``An automata-theoretic approach to automatic
  program verification,'' in \emph{1st Symposium on Logic in Computer
  Science}.\hskip 1em plus 0.5em minus 0.4em\relax IEEE Computer Society, 1986,
  pp. 322--331.

\bibitem{karaman2011sampling}
S.~Karaman and E.~Frazzoli, ``Sampling-based algorithms for optimal motion
  planning,'' \emph{The international journal of robotics research}, vol.~30,
  no.~7, pp. 846--894, 2011.

\bibitem{mayne2011tube}
D.~Q. Mayne, E.~C. Kerrigan, E.~Van~Wyk, and P.~Falugi, ``Tube-based robust
  nonlinear model predictive control,'' \emph{International journal of robust
  and nonlinear control}, vol.~21, no.~11, pp. 1341--1353, 2011.

\bibitem{everett2021neural}
M.~Everett, ``Neural network verification in control,'' in \emph{IEEE
  Conference on Decision and Control (CDC)}, 2021, pp. 6326--6340.

\end{thebibliography}
\end{document}